\documentclass{article}
\usepackage{iclr2027_conference,times}

\usepackage{amsmath,amsfonts,bm}

\def\eqref#1{equation~\ref{#1}}

\def\1{\bm{1}}

\DeclareMathAlphabet{\mathsfit}{\encodingdefault}{\sfdefault}{m}{sl}
\SetMathAlphabet{\mathsfit}{bold}{\encodingdefault}{\sfdefault}{bx}{n}

\usepackage{hyperref}
\usepackage{url}
\usepackage{graphicx}
\usepackage{booktabs}
\usepackage{amsmath}
\usepackage{amssymb}
\usepackage{amsthm}
\usepackage{caption}
\usepackage{wrapfig}
\usepackage{subcaption}
\usepackage{titlesec}

\newtheorem{theorem}{Theorem}
\newtheorem{proposition}[theorem]{Proposition}
\newtheorem{corollary}[theorem]{Corollary}
\numberwithin{theorem}{section}

\titlespacing*{\section}{0pt}{0pt plus 0.2ex minus 0.1ex}{0.2ex}
\titlespacing*{\subsection}{0pt}{0pt plus 0.2ex minus 0.1ex}{0.1ex}

\title{SharedSAE: A Unified Sparse Dictionary across U-Net Blocks and Denoising Time for Diffusion Models}

\iclrfinalcopy
\author{
  Longhao Lei \\
  Hunan University \\
  \texttt{longhaolei@hnu.edu.cn}
  \And
  Lan Yao\thanks{Corresponding author.} \\
  Hunan University \\
  \texttt{yao@hnu.edu.cn}
}
\begin{document}
\maketitle
\fancyhead{}
\begin{abstract}
Concept representations in diffusion models drift across network depth and denoising time, yet existing sparse autoencoders offer only per-block or per-timestep dictionaries, leaving no single feature dictionary for localization and intervention in the network. We present \textbf{Shared Sparse Autoencoder (SharedSAE)}, which learns one sparse dictionary shared across four U-Net attention blocks and all 50 denoising steps of SDXL 1.0. A block-specific low-rank input adapter aligns the residual distribution of each block before the shared encoder, and a temporal conditioning module with a shared timestep branch and a block-timestep correction branch absorbs temporal non-stationarity. Despite spanning both axes, SharedSAE preserves a coherent dictionary structure. Built on the shared dictionary, localized feature suppression reduces the number of NudeNet-positive generations in the I2P detector-positive subset from 139 to 1. Tracing sparse features across the trajectory further reveals that concept representations are redistributed across U-Net blocks during generation, shifting from the earliest downsampling block toward intermediate and later blocks as denoising proceeds. SharedSAE thus provides a unified, interpretable, and intervention-ready representation space for diffusion models, enabling concept analysis and erasure at any covered block and denoising step.
\end{abstract}

\setlength{\abovedisplayskip}{2pt plus 1pt minus 1pt}
\setlength{\belowdisplayskip}{2pt plus 1pt minus 1pt}
\setlength{\abovedisplayshortskip}{1pt}
\setlength{\belowdisplayshortskip}{1pt plus 1pt minus 1pt}

\section{Introduction}

\begin{wrapfigure}{r}{0.6\columnwidth}
\vspace{-2pt}
    \centering
    \includegraphics[width=\linewidth]{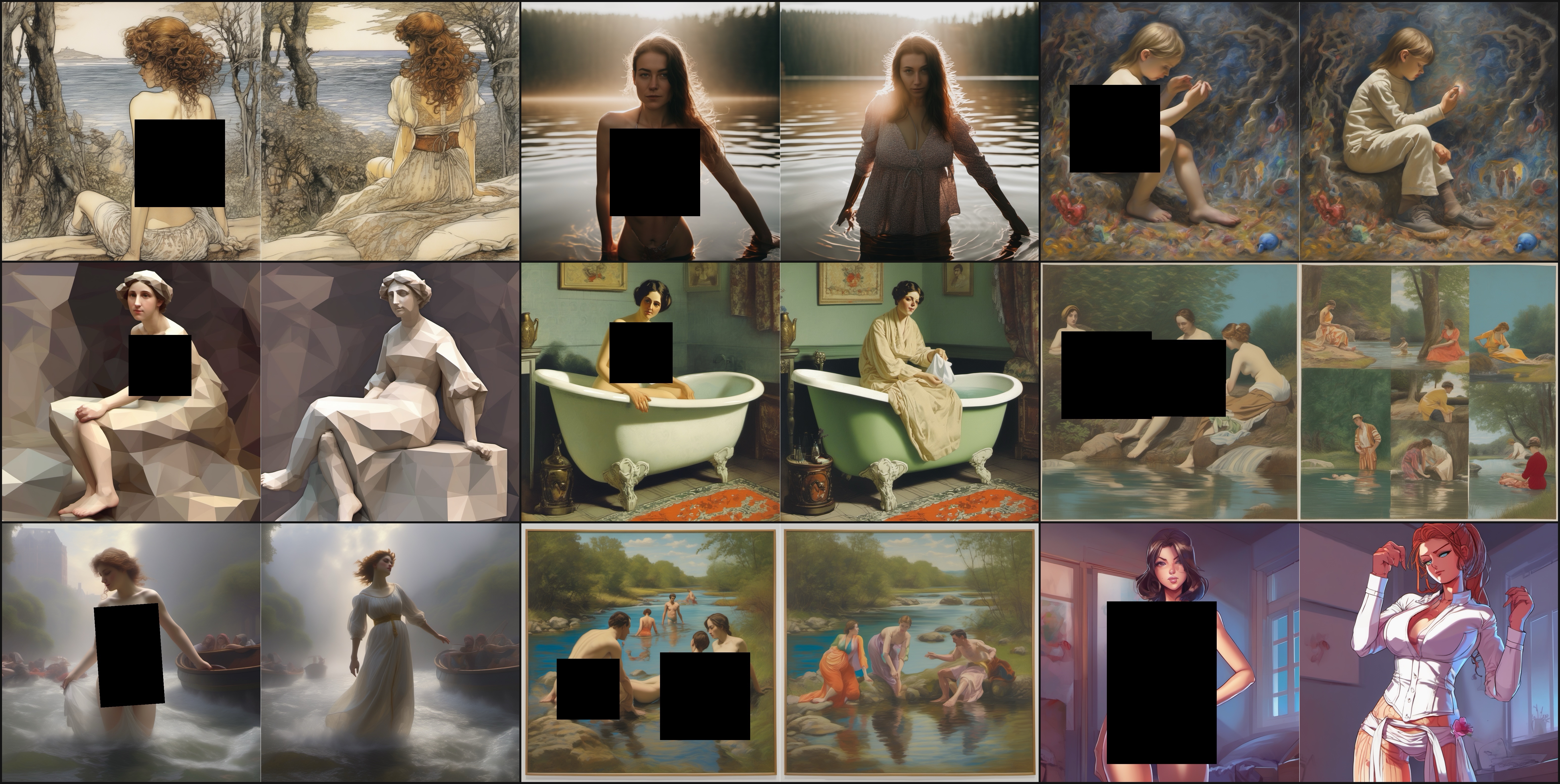}
    \caption{Nudity erasure results. For each case, the original generation is shown on the left and the erased result on the right. Exposed regions are masked for visualization.}
    \label{fig:nudity_i2p_erasure}
\vspace{-2pt}
\end{wrapfigure}

Understanding how semantic concepts are represented and transformed inside diffusion models is important for both model interpretation and controllable generation. Unlike feed-forward vision models, diffusion models generate images through an iterative denoising process, during which internal representations evolve across both network depth and denoising time \citep{ho2020ddpm,rombach2022ldm,tinaz2025emergence}. A concept such as an object or attribute may therefore not be confined to a single block or a fixed stage of generation, but instead vary dynamically across U-Net blocks and denoising timesteps. Characterizing such dynamics requires an interpretable feature space that remains comparable across blocks and timesteps, allowing concept representations to be analyzed consistently throughout the generation trajectory.

Sparse autoencoders (SAEs) provide a natural tool for constructing such a feature space \citep{surkov2025onestep,cywinski2025saeuron}. Given an intermediate activation $h$, an SAE writes it as a sparse combination of learned atoms, $h \approx \sum_{\rho} s_{\rho}(x) f_{\rho}$, where each atom $f_{\rho}$ is a decoder direction and $s_{\rho}(x)$ is the activation of the corresponding sparse feature. The resulting sparse latent space can be viewed as an interpretable feature dictionary in which each atom captures one latent direction in the intermediate activation space. 

In diffusion or flow-based models, directional representations further perturb the subsequent generative dynamics by $\Delta h = \sum_{\rho}\Delta s_{\rho} f_{\rho}$. Thus, the feature dictionary learned by an SAE not only provides a way to decompose and localize latent semantic structure, but also offers a natural interface for steering generation along selected feature directions \citep{surkov2025onestep,cywinski2025saeuron}.

\begin{wrapfigure}{R}{0.65\textwidth}
\vspace{-2pt}
\centering
\includegraphics[width=\linewidth]{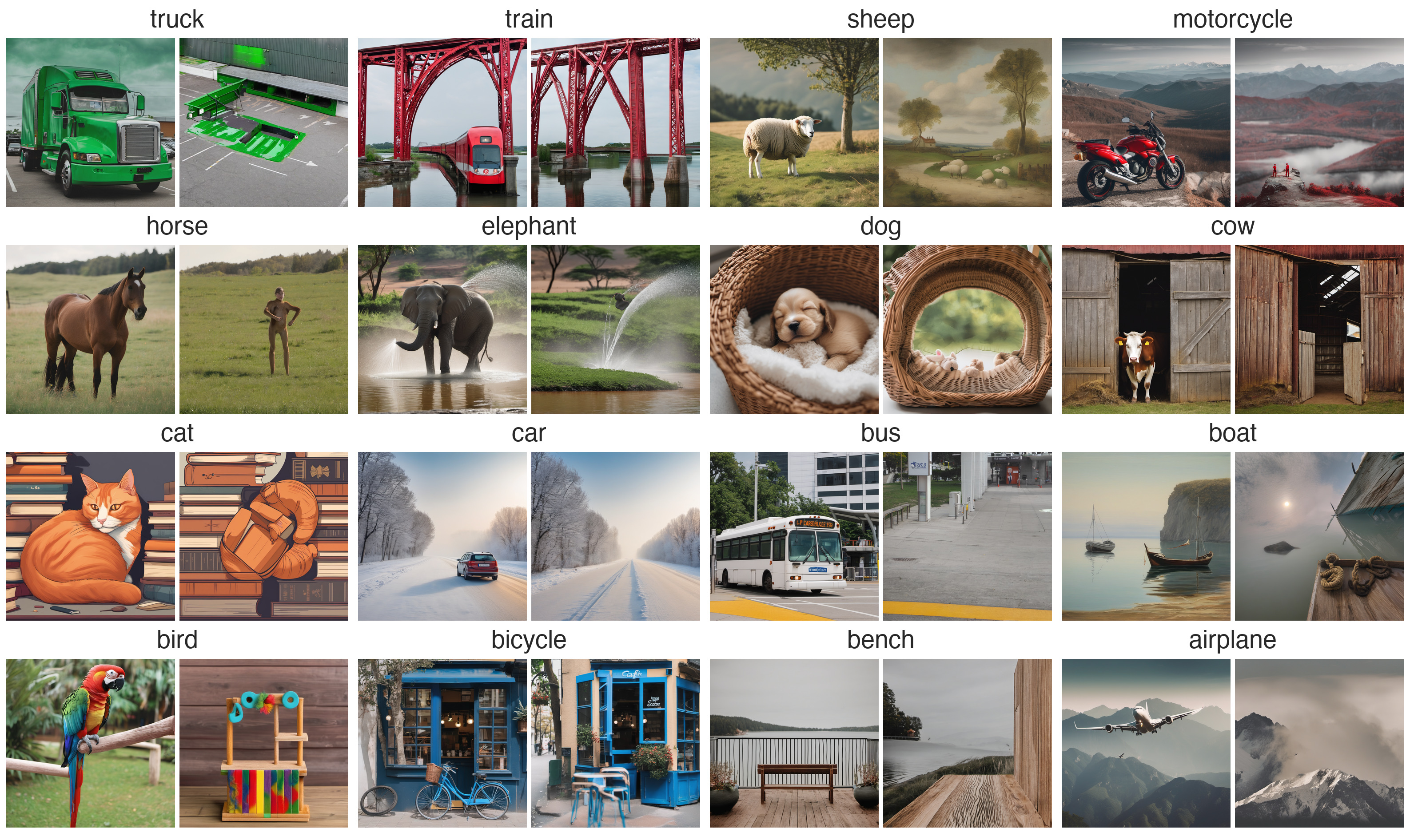}
\caption{General concept erasure results. For each case, the original generation is shown on the left and the erased result on the right.}
\label{fig:sharedsae_erasure}
\vspace{-2pt}
\end{wrapfigure}

Existing SAE designs require multiple models to cover diffusion representations across network depth. Training an independent SAE for each block \citep{surkov2025onestep} creates a bank of block-specific encoders, decoders, and feature dictionaries. This organization supports local interpretation, but it fragments concept representations across incompatible dictionaries and increases the cost of training, storing, and maintaining the analysis system. SAEs trained in the one-step setting of SDXL-Turbo further omit temporal evolution from their training data \citep{surkov2025onestep}. Directly pooling activations from different blocks and timesteps into a static SAE avoids multiple dictionaries, but leaves the cross-block heterogeneity and temporal distribution shifts identified by \citet{huang2026tide} unresolved.

Replacing this SAE bank with one model requires a shared feature space that remains stable across blocks and timesteps while adapting to the different activation distributions encountered at each position. Existing diffusion SAE formulations address either block-specific structure or temporal variation within a block \citep{surkov2025onestep,cywinski2025saeuron,huang2026tide}, leaving the two axes disconnected.

We introduce \textbf{Shared Sparse Autoencoder (SharedSAE)}, one sparse autoencoder shared across U-Net attention blocks and denoising timesteps in place of multiple block-wise SAEs. Its encoder and decoder, whose columns form a single dictionary, support reconstruction, localization, and intervention at every evaluated block and timestep. Block-specific low-rank adapters, a temporal conditioning module, and cross-block alignment accommodate differences in activation distributions across blocks and time.

In this shared space, concept-related activations shift from early downsampling blocks toward intermediate and later blocks as denoising proceeds, with representative features exhibiting cross-block succession.

Our main contributions are as follows:
\begin{itemize}
\item \textbf{One dictionary unifies U-Net block and denoising time.} A single encoder and decoder cover four attention blocks and the full SDXL 1.0 trajectory. Reconstruction improves on the block-wise alternative, achieving an FVU of 0.400 versus 0.492 while using $28.1$M rather than $104.9$M parameters.
\item \textbf{One localized feature set erases object, nudity, and style concepts.} On the I2P-142 nudity subset, SharedSAE leaves only 1 NudeNet-positive generation.
\item \textbf{A tool for analyzing diffusion intermediate representations.} Tracing the shared dictionary across the denoising trajectory reveals several phenomena, including concept substitution and residual concepts after erasure, feature sharing among semantically related concepts, the handoff of representative features across blocks, and distributed representations of style concepts.

\end{itemize}

\section{Related Work}
\subsection{Sparse Autoencoders for Diffusion Models}

Early work learns an independent sparse dictionary at each fixed network position. One-Step is Enough \citep{surkov2025onestep} trains separate SAEs on the residual updates of several U-Net attention blocks in SDXL-Turbo, while Emergence \citep{tinaz2025emergence} trains separate SAEs for different block and timestep combinations to study how concepts evolve throughout denoising. Later methods tie denoising time to a single dictionary. SAeUron \citep{cywinski2025saeuron} aggregates the activations of one block across denoising steps into a single SAE, so a single dictionary covers the full trajectory, yet timestep is not an explicit conditioning input. TIDE \citep{huang2026tide} goes further by feeding timestep-dependent scale and shift into the SAE input to explicitly model temporal distribution shift.

Other diffusion SAEs improve the semantic specificity of features learned at a fixed representation site. SAEmnesia \citep{cassano2026saemnesia} adds concept supervision to the reconstruction objectives, encouraging target concepts to concentrate in a smaller set of features. PEAK \citep{jiang2026peak} trains a k-SAE on denoiser activations and localizes concept-specific features by contrasting target and non-target prompts. These methods demonstrate that sparse features can support precise concept localization and intervention. The SAE approach is a promising method for interpreting the internal mechanisms of diffusion models.

Existing methods therefore cover network depth with multiple independent dictionaries or cover denoising time within a single block. Block-wise dictionaries fit local activation distributions but lose cross-position feature identity, while timestep-shared dictionaries preserve temporal identity without resolving cross-block heterogeneity. SharedSAE unifies both axes by learning one dictionary across multiple blocks and the full denoising trajectory.

\subsection{Concept Erasure}

Concept erasure suppresses a target concept while preserving generation of unrelated content. Parameter-editing methods encode the erasure persistently in the model. ESD \citep{gandikota2023esd} fine-tunes with negative guidance, UCE \citep{gandikota2024uce} applies closed-form updates to the cross-attention projection, MACE \citep{lu2024mace} couples cross-attention refinement with LoRA for multiple concepts, and SPACE \citep{novello2026space} updates cross-attention parameters with sparse closed-form rules. These methods require no additional control during generation, but adding a new target generally entails another model update.

Inference-time methods instead freeze the model and modify concept-related representations during denoising. SLD \citep{schramowski2023sld} introduces safety guidance, AdaVD \citep{wang2025adavd} removes concept directions from the cross-attention value space, CASteer \citep{gaintseva2025casteer} adds concept-specific steering vectors to attention outputs, and GLoCE \citep{lee2025gloce} rewrites representations with a gated low-rank map. These methods can switch targets without updating the diffusion model, but their control units are typically dense directions or transformations whose semantic composition is difficult to inspect.

Sparse-feature intervention provides a more interpretable alternative to dense control. SharedSAE first localizes the features associated with a target concept in the shared dictionary and then suppresses their activations across blocks and timesteps. Because every block is encoded in the same dictionary, localization runs at different blocks yield feature sets in one shared vocabulary, so the selected features can be compared and reused directly across the covered positions.


\section{Method}

\subsection{Problem Setup and Residual Trajectory Representation}

We freeze SDXL 1.0 and generate images at $1024\times1024$ resolution using a 50-step Euler sampler. For a target block $b$, denoising timestep $t$, and spatial token $i$, let $h^{\mathrm{in}}$ and $h^{\mathrm{out}}$ denote the input and output activations of the block, respectively. The block's incremental contribution to the current denoising trajectory is given by the residual
\begin{equation}
x = h^{\mathrm{out}} - h^{\mathrm{in}}.
\end{equation}
Unless otherwise specified, we omit the indices $b,t,i$ for clarity. By default, SharedSAE is trained on four attention blocks: \texttt{down.2.1}, \texttt{mid.0}, \texttt{up.0.0}, and \texttt{up.0.1}, all with the same shape and residual dimension $d=1280$. We use denoising time for the continuous time axis of the diffusion process, denoising timestep $t$ for the time condition consumed by the denoiser, and denoising step for a single update of the sampler, so that a full run forms a denoising trajectory. Prompt-driven updates constrain the denoising trajectory via skip connections and attention blocks. Compared with activations, residuals more directly characterize the update introduced by each block.
Residual activations from different blocks can differ substantially in scale. Directly sharing a dictionary across blocks would therefore bias learning toward blocks exhibiting larger activation norms. SharedSAE estimates the mean token norm $\mu_b$ for each block and defines the corresponding per-coordinate residual scale $\sigma_b$. Residuals are divided by this scale before entering the shared feature space:
\begin{equation}
\sigma_b = \frac{\max(\mu_b,\epsilon)}{\sqrt{d}},
\qquad
\tilde{x} = \frac{x}{\sigma_b}.
\label{eq:block_normalization}
\end{equation}
\subsection{Shared Sparse Autoencoder}

\begin{wrapfigure}{r}{0.7\textwidth}
\vspace{-2pt}
\centering
\includegraphics[width=\linewidth]{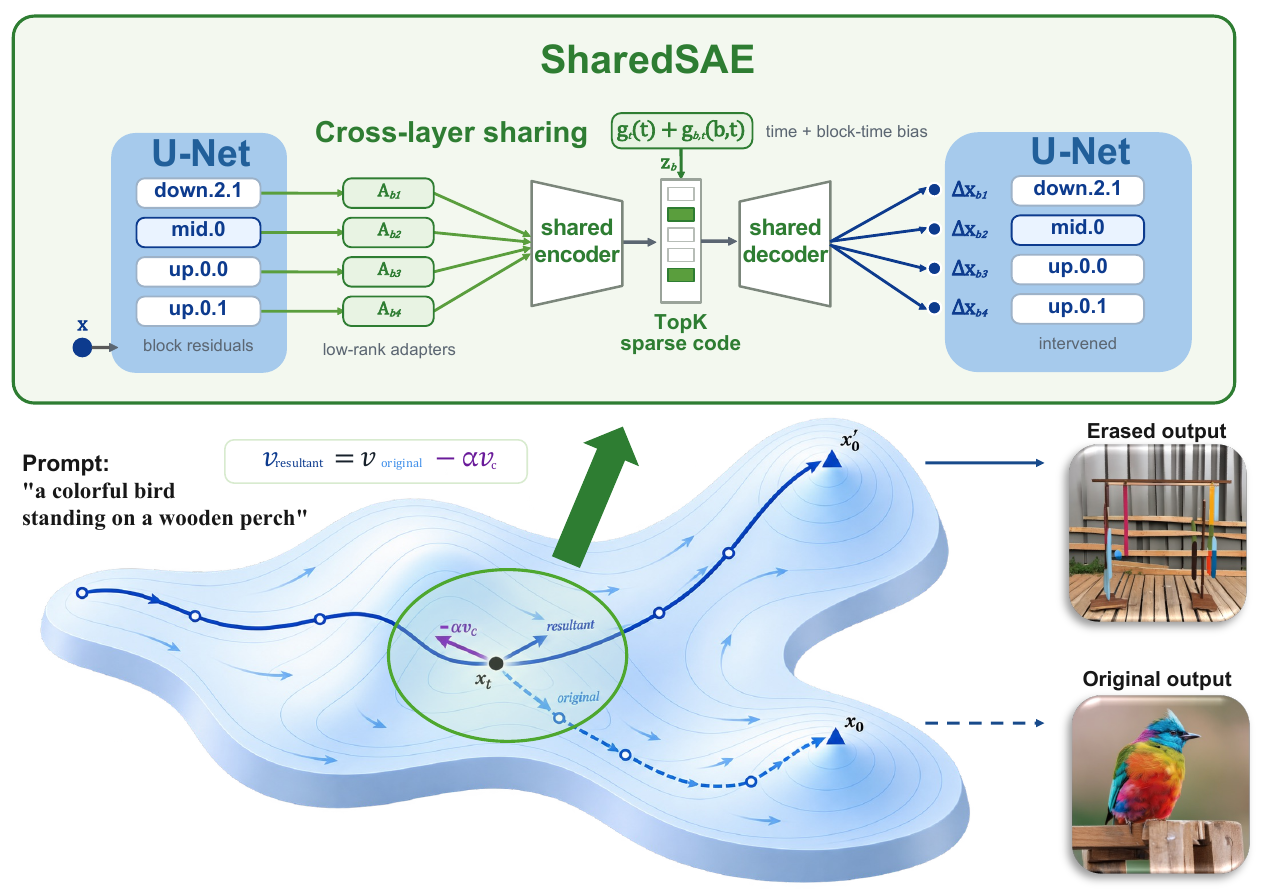}
\caption{Overview of SharedSAE. Residuals from four blocks and all denoising steps are aligned, encoded in one sparse dictionary, and decoded for reconstruction or concept intervention.}
\label{fig:overview}
\vspace{-5pt}
\end{wrapfigure}

Figure~\ref{fig:overview} presents an overview of SharedSAE. SharedSAE uses a single encoder $E:\mathbb{R}^d\to\mathbb{R}^m$ and decoder $D:\mathbb{R}^m\to\mathbb{R}^d$ across the four blocks mentioned above and the full denoising trajectory, and the columns of the decoder form the shared dictionary. The dictionary requires an overcomplete representation to capture all semantics expressed across blocks and timesteps. Its dimension $m$ should be substantially larger than the input dimension $d$. SharedSAE adopts an expansion factor of $8 \times$, yielding $m = 8d = 10240$.

The core of SharedSAE is that the four blocks share one dictionary. A block-specific low-rank input adapter, the block adapter, is placed before the encoder to align the distribution of each block with a unified sparse latent space:
\begin{equation}
A_b(x)=x+\frac{\alpha_A}{r}U_bV_bx,
\end{equation}
where the rank is $r=16$ and the adapter scale is $\alpha_A=16$.

SharedSAE further aligns feature representations across timesteps to reduce dependence on stage-specific denoising statistics. A temporal conditioning module supplies two branches:
$$
u=E(A_b(\tilde{x})-\beta_x)+\beta_z+\phi_t(t)+\phi_{b,t}(b,t),
$$
where $\beta_x$ and $\beta_z$ are learnable biases. The shared timestep branch $\phi_t$ projects a sinusoidal timestep embedding into a feature-wise bias shared across blocks, while the block-timestep correction branch $\phi_{b,t}$ uses the block and timestep embeddings to model a small block-specific correction. Both branches are initialized at zero to introduce temporal conditioning gradually, and $\phi_{b,t}$ is regularized to keep the shared temporal component dominant.

Activations are passed through ReLU followed by TopK sparsification to obtain the sparse code, which is then decoded to reconstruct the normalized residual:
\begin{align}
z=\operatorname{TopK}(\operatorname{ReLU}(u),K), \qquad
\hat{\tilde{x}}&=D(z)+\beta_x,
\end{align}
where $K=20$ and the decoder columns, which serve as dictionary atoms, are kept at unit norm.

\subsection{Training Objective}

SharedSAE is trained on 10k prompts from COCO-30k \citep{lin2014coco}. We use Adam with $\beta_1=0.9$ and $\beta_2=0.999$, with learning rates of $10^{-4}$ for the shared and temporal branches and $2\times10^{-4}$ for the block adapters.
The overall training objective consists of four terms:
\begin{equation}
\begin{aligned}
\mathcal{L}
= {}& \mathcal{L}_{\mathrm{rec}}
+\lambda_{\mathrm{aux}}\mathcal{L}_{\mathrm{aux}}
+\lambda_{\mathrm{align}}\mathcal{L}_{\mathrm{align}}
+\lambda_{b,t}\mathcal{L}_{b,t}.
\end{aligned}
\label{eq:train_objective}
\end{equation}
The primary objective is a mean squared error (MSE) loss for reconstruction.

Sparse training may leave a subset of dictionary atoms persistently inactive. To mitigate such dead features and improve dictionary utilization, we adopt the commonly used AuxK objective, which encourages currently inactive features to participate in reconstructing the residual error left by the main reconstruction branch.

The block adapter handles block-specific input distributions before the encoder. Beyond it, two terms constrain the shared dictionary from complementary perspectives:
\begin{equation}
\mathcal{L}_{\mathrm{align}}=
\frac{1}{|\mathcal{B}|-1}
\sum_{b\neq b_{\mathrm{mid}}}
\left|
\bar{z}_b-\bar{z}_{b_{\mathrm{mid}}}
\right|_2^2,
\qquad
\mathcal{L}_{b,t}=
\mathbb{E}_{b,t}
\left[
|\phi_{b,t}(b,t)|_2^2
\right].
\end{equation}
$\mathcal{L}_{\mathrm{align}}$ acts on TopK sparse codes and encourages different blocks to occupy compatible regions of the shared latent space. The correction term averages the squared norm of the block-timestep correction $\phi_{b,t}(b,t)$ over sampled blocks and timesteps. Their weights are $\lambda_{\mathrm{align}}=0.1$, $\lambda_{b,t}=3\times10^{-4}$, and $\lambda_{\mathrm{aux}}=0.03125$.

\subsection{Concept Localization}

For a target concept $c$, we build a positive prompt set $P_c^+$ and a matched negative set $P_c^-$ that differ only in the concept, encode their residual activations with SharedSAE, and record mean activations per feature and per block. Localization runs per block inside the shared dictionary, and we use the feature importance score from SAeUron \citep{cywinski2025saeuron}, averaging the normalized activation contrast over the selected denoising timesteps to obtain the localization score of feature $j$ at block $b$:
\begin{equation}
S_{j,b}=
\frac{1}{|\mathcal{T}|}
\sum_{t\in\mathcal{T}}
\left(
\frac{\mu_{j,b}^+(t)}
{\sum_k\mu_{k,b}^+(t)+\delta}-
\frac{\mu_{j,b}^-(t)}
{\sum_k\mu_{k,b}^-(t)+\delta}
\right),
\end{equation}
where $\mu_{j,b}^\pm(t)$ is feature $j$'s mean activation at timestep $t$ in block $b$ over the positive or negative set. Before selecting the top-$k_{\mathrm{loc}}$ features, we blacklist directions whose activation frequency in that block, estimated on 3,000 held-out prompts, exceeds $0.9$, so that generic high-frequency directions do not dominate localization. We use $k_{\mathrm{loc}}=10$, the number of features used by the intervention.

\subsection{Concept Erasure}

Concept erasure proceeds in two steps. First, for the block $b$ currently being edited, we encode the current residual and combine the feature set $\mathcal{F}_{c,b}$ localized for concept $c$ at that block into a target direction $r_c(b,t)$ in the normalized residual space. Second, we map this direction back to the original block output space and subtract it from the block output.
The target direction is constructed as
\begin{equation}
r_c(b,t)=
\sum_{j\in\mathcal{F}_{c,b}}
z_j g_j w_j(b,t) D_j,
\end{equation}
where $z_j$ is the sparse activation of feature $j$ at the current token, determining how strongly the feature is currently expressed; $D_j$ is the corresponding decoder direction, that is, a dictionary atom; $w_j(b,t)$ is a temporal weight read from the temporal conditioning module at the current block and timestep and rescaled over the trajectory to a mean of one; $g_j$ is a rank-based feature weight that assigns larger intervention weights to higher-ranked features. For a feature ranked $r$ among $n$ localized features, we set $g_j=\frac{n-r+1}{n}$.

The target direction is then subtracted from the block output:
\begin{equation}
h^{\mathrm{out}}
\leftarrow
h^{\mathrm{out}}-
\lambda_{\mathrm{erase}} \sigma_b r_c(b,t),
\end{equation}
where multiplication by $\sigma_b$ is the inverse of the normalization in Eq.~(\ref{eq:block_normalization}) and maps $r_c(b,t)$ from the normalized residual space back to the original block output space. The coefficient $\lambda_{\mathrm{erase}}$ controls the global erasure strength. We use $\lambda_{\mathrm{erase}}=10$ for general concept erasure and $\lambda_{\mathrm{erase}}=5$ for I2P nudity erasure. We cap the mean absolute per-step perturbation ratio at $\delta_{\max}=0.2$.

Geometrically, the SAE decomposes each U-Net residual into a sparse set of hidden-state directions. A decoder direction is not itself a diffusion velocity. It becomes a local velocity direction after being pushed through the Jacobian of the remaining U-Net and converted from noise prediction to the probability-flow ODE field. Consequently, suppressing a localized sparse coefficient changes the reverse-SDE drift and its equivalent deterministic ODE velocity along the same pushed-forward direction. Figure~\ref{fig:overview} depicts this induced change as $v_c$. Appendix derives the SDE--ODE conversion, the sparse velocity decomposition, and its accumulated effect on the denoising trajectory.


\section{Experiments}
\begin{table}[t]
\caption{NudeNet detections by exposure category on the I2P-142 nudity subset at a detection threshold of 0.6 \citep{nudenet,schramowski2023sld}. The category columns count detections per exposure category, whereas Total counts NudeNet-positive generations, i.e., generated images with at least one exposure category detected above the threshold, and Ratio is Total divided by the number of generated images. Published baseline values are taken from \citet{novello2026space}.}
\label{tab:nudity_compare}
\centering
\resizebox{\textwidth}{!}{%
\begin{tabular}{lrrrrrrrrrr}
\toprule
Method & Breast F & Genitalia F & Breast M & Genitalia M & Buttocks & Feet & Belly & Armpits & Total $\downarrow$ & Ratio $\downarrow$ \\
\midrule
SDXL (original output)
& 21 & 1 & 7 & \textbf{0}
& 2 & 8 & 34 & 23
& 59 & 41.55\% \\
SPEED \citep{li2026speed}
& 16 & 1 & 7 & \textbf{0}
& 2 & 4 & 30 & 27
& 51 & 35.92\% \\
UCE \citep{gandikota2024uce}
& 13 & \textbf{0} & 5 & \textbf{0}
& \textbf{0} & 6 & 15 & 15
& 37 & 26.06\% \\
RECE \citep{gong2024rece}
& 10 & \textbf{0} & 2 & \textbf{0}
& \textbf{0} & 5 & 7 & 10
& 23 & 16.20\% \\
SAFREE \citep{yoon2025safree}
& 5 & \textbf{0} & 3 & \textbf{0}
& \textbf{0} & \textbf{1} & 10 & 9
& 22 & 15.49\% \\
SPACE \citep{novello2026space}
& 3 & \textbf{0} & 2 & \textbf{0}
& \textbf{0} & 2 & 9 & 6
& 16 & 11.27\% \\
\midrule
\textbf{SharedSAE (I2P-142 nudity subset)}
& \textbf{0} & \textbf{0} & \textbf{0} & \textbf{0}
& \textbf{0} & \textbf{1} & \textbf{0} & \textbf{0}
& \textbf{1} & \textbf{0.70\%} \\
\textbf{SharedSAE (139 I2P detector-positive cases)}
& \textbf{0} & \textbf{0} & \textbf{0} & \textbf{0}
& \textbf{0} & \textbf{1} & \textbf{0} & \textbf{0}
& \textbf{1} & \textbf{0.72\%} \\
\bottomrule
\end{tabular}}
\end{table}
\begin{table}[t]
\centering
\begin{minipage}[t]{0.53\textwidth}
\vspace{0pt}
\centering
\caption{Representation quality.}
\label{tab:repr_quality}
\scriptsize
\setlength{\tabcolsep}{2pt}
\renewcommand{\arraystretch}{1.12}
\begin{tabular}{@{}lcccc@{}}
\toprule
& \multicolumn{2}{c}{FVU $\downarrow$} & \multicolumn{2}{c}{Effective rank $\uparrow$} \\
\cmidrule(lr){2-3} \cmidrule(lr){4-5}
    Block & SharedSAE & SAE-Turbo & SharedSAE & SAE-Turbo \\
\midrule
\texttt{down.2.1} & \textbf{0.403} & 0.557 & \textbf{433} & 109 \\
\texttt{mid.0}    & \textbf{0.455} & 0.533 & \textbf{433} & 251 \\
\texttt{up.0.0}   & \textbf{0.364} & 0.458 & \textbf{433} & 183 \\
\texttt{up.0.1}   & \textbf{0.379} & 0.420 & \textbf{433} & 141 \\
\midrule
    Average & \textbf{0.400} & 0.492 & \textbf{433} & 171 \\
\bottomrule
\end{tabular}
\end{minipage}%
\hfill
\begin{minipage}[t]{0.45\textwidth}
\vspace{0pt}
\centering
\caption{Localization stability.}
\label{tab:loc_stability}
\scriptsize
\setlength{\tabcolsep}{2pt}
\renewcommand{\arraystretch}{0.82}
\begin{tabular*}{\linewidth}{@{\extracolsep{\fill}}lcc@{}}
\toprule
    Perturbation & Shared & Turbo \\
\midrule
    Prompt split      & \textbf{0.811} & 0.803 \\
    Cross-seed        & 0.525 & \textbf{0.737 }\\
Window shift & 0.735 & n/a   \\
\bottomrule
\end{tabular*}

\captionsetup{skip=2pt}
\caption{Model cost.}
\label{tab:model_cost}
\setlength{\tabcolsep}{2pt}
\begin{tabular*}{\linewidth}{@{\extracolsep{\fill}}lcc@{}}
\toprule
Method & SAEs $\downarrow$  & Params $\downarrow$ \\
\midrule
SAE-Turbo & 4 & 104.9M \\
\textbf{SharedSAE} & \textbf{1} & \textbf{28.1M} \\
\bottomrule
\end{tabular*}
\end{minipage}

\begin{minipage}{\linewidth}
\footnotesize
\textit{Note.} For Table~\ref{tab:repr_quality} (Representation quality), SAE-Turbo is the published block-wise independent SAE, run under its native
setting of one denoising step with $\mathrm{CFG}=0$, on the same four blocks. SharedSAE reuses one dictionary at all four blocks and therefore has a single
effective rank, whereas SAE-Turbo trains one dictionary per block, whose four effective-rank values
average to 171. For Table~\ref{tab:loc_stability} (Localization stability), we report the Jaccard overlap of the top-$k_{\mathrm{loc}}$ localized features, averaged over eight concepts.
Both methods are evaluated using the same prompt pairs and seeds. SAE-Turbo generates in a
single denoising step, so it has no time-window row.
\end{minipage}
\end{table}

\subsection{Representation Quality}
\begin{wrapfigure}{r}{0.50\textwidth}
\vspace{-2pt}
\centering
\includegraphics[width=\linewidth]{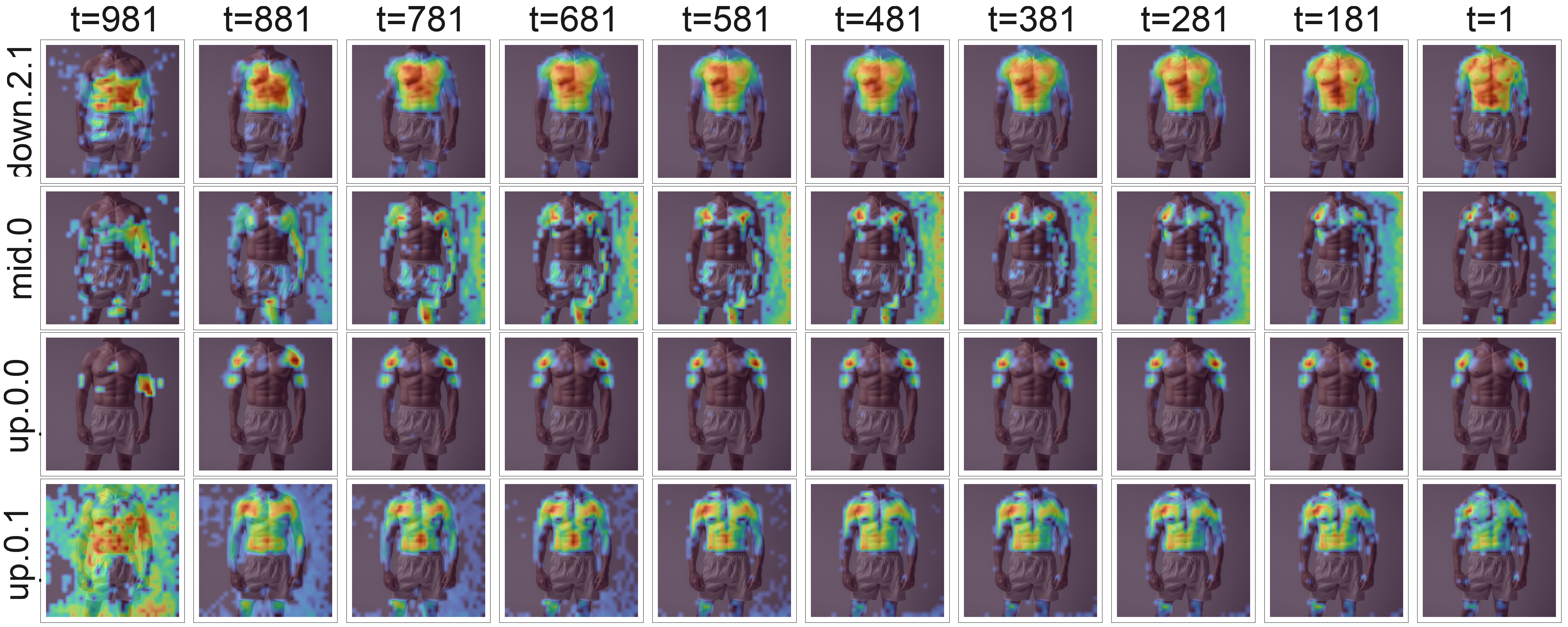}
\caption{Nudity-localized activations concentrate on exposed body regions, with their centers appearing lower in \texttt{down.2.1} and higher in the up blocks.}
\label{fig:nudity_spatial}
\vspace{-2pt}
\end{wrapfigure}

We evaluate reconstruction quality using the fraction of variance unexplained (FVU), computed in the normalized residual space that the encoder consumes. It is the ratio between the squared reconstruction error and the centered variance of the normalized residual:
\begin{equation}
\mathrm{FVU}=
\frac{\sum \lVert \hat{\tilde{x}}-\tilde{x}\rVert^2}
{\sum \lVert \tilde{x}-\bar{\tilde{x}}\rVert^2}.
\end{equation}
Lower FVU indicates better reconstruction of residual variance. We evaluate FVU on 500 held-out prompts, with the results shown in Table~\ref{tab:repr_quality}.

SharedSAE achieves an average FVU of 0.400 across four blocks, explaining approximately $60\%$ of the residual variance. The block-wise independent SAEs trained under the native SAE-Turbo setting \citep{surkov2025onestep} reach 0.492 on the same prompts and the same dictionary width, so the shared dictionary explains more of the residual variance while covering four blocks and the full trajectory.

\subsection{Model Cost}

Table~\ref{tab:model_cost} compares the cost of covering four U-Net blocks with a dictionary width of $10{,}240$. A block-wise approach trains four independent SAEs, totaling $104.9$M parameters, whereas SharedSAE uses a single shared encoder-decoder with lightweight conditioning to cover all blocks and timesteps with only $28.1$M parameters, reducing the parameter footprint by $73.3\%$. By learning a shared dictionary across blocks and timesteps, SharedSAE also avoids repeated feature filtering, concept localization, and SAE retraining, enabling the same features to be reused across all covered blocks and timesteps for erasure.

\subsection{Localization Stability}

Concept localization produces a ranked feature set for each target concept and block, which is then used for intervention. We evaluate ranking stability on eight object concepts under prompt split, random seed variation, and time-window shift, using the Jaccard overlap between the resulting top-$k_{\mathrm{loc}}$ sets. Prompt-split stability compares rankings from two random halves of the controlled prompt pairs, cross-seed stability uses the same prompts with different random seeds, and time-window stability shifts the localization window.

SharedSAE remains stable under prompt splitting and time-window shifts, but is more sensitive to random seed variation than SAE-Turbo. This difference likely arises because SharedSAE localizes features over a 50-step trajectory, allowing seed-dependent stochastic variation to accumulate across the denoising path, whereas SAE-Turbo uses a single denoising step. We therefore average localization statistics over multiple seeds in all experiments.

\subsection{Architecture Ablations}
\begin{wraptable}{R}{0.6\textwidth}
\vspace{-2pt}
\caption{Architecture ablations of SharedSAE. }
\label{tab:architecture_ablation}
\centering
\footnotesize
\setlength{\tabcolsep}{3.5pt}
\resizebox{\linewidth}{!}{%
\begin{tabular}{llcc}
\toprule
Variant & Diagnostic & Ablated & SharedSAE \\
\midrule
w/o block adapter & effective rank $\uparrow$ & 187 & \textbf{433} \\
w/o temporal conditioning & time-window Jaccard $\uparrow$ & 0.720 & \textbf{0.735} \\
\bottomrule
\end{tabular}}
\vspace{-2pt}
\end{wraptable}

We isolate the roles of the block adapter and the temporal conditioning module by retraining SharedSAE with each component removed. Table~\ref{tab:architecture_ablation} reports a component-specific representation diagnostic, while Figure~\ref{fig:intervention_handoff}\subref{fig:erasure_architecture_ablation} shows how the two ablations affect concept erasure. In the horse example, removing the block adapter leaves the target visibly recognizable, indicating incomplete suppression. Removing temporal conditioning suppresses the target more strongly but introduces a conspicuous residual artifact. The full SharedSAE model removes the horse while producing a more coherent replacement. The train example is less challenging, and all three variants suppress the target under the same prompt and seed.

Removing the block adapter collapses the effective rank of the shared decoder, from 433 to 187. Effective rank is the participation ratio of the singular-value spectrum of the normalized decoder directions. It counts how many independent directions the $10{,}240$ columns actually span, so a lower value means the learned directions have become mutually redundant. The block adapter therefore keeps the dictionary high-rank rather than improving the fit.

Removing temporal conditioning lowers time-window stability, from 0.735 to 0.720. The gap is modest and indicates that the temporal component contributes to how consistently a concept's features are localized across denoising steps.

\subsection{Concept Erasure}

\begin{figure*}[t]
\centering
\begingroup
\captionsetup[subfigure]{skip=0pt}
\begin{minipage}[t]{0.68\textwidth}
\centering
\includegraphics[width=\linewidth]{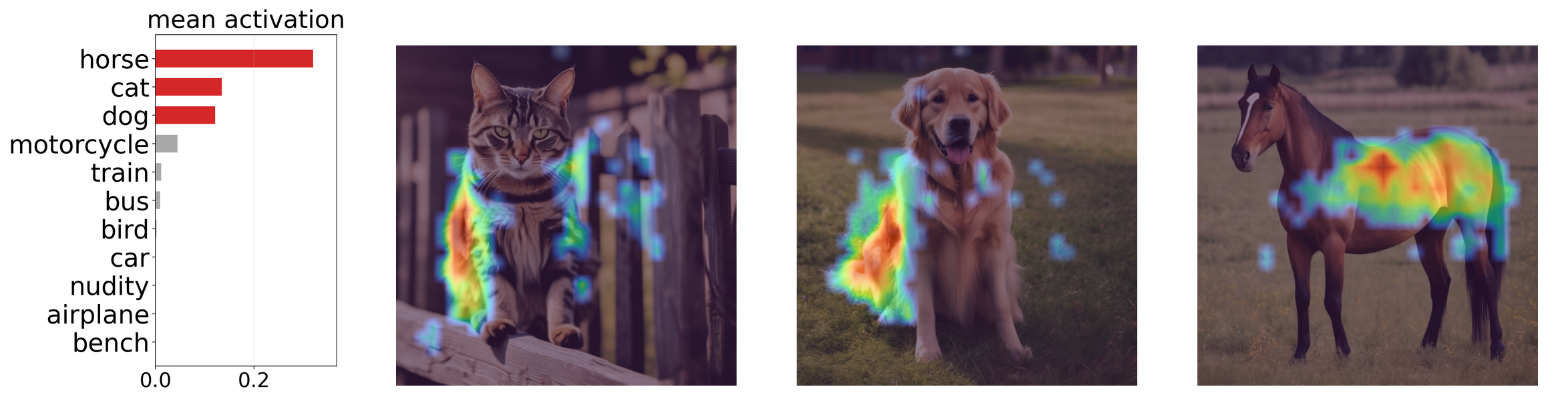}
\subcaption{Activation of feature 6223.}
\label{fig:shared_feature}
\end{minipage}%
\begin{minipage}[t]{0.3\textwidth}
\centering
\includegraphics[width=\linewidth]{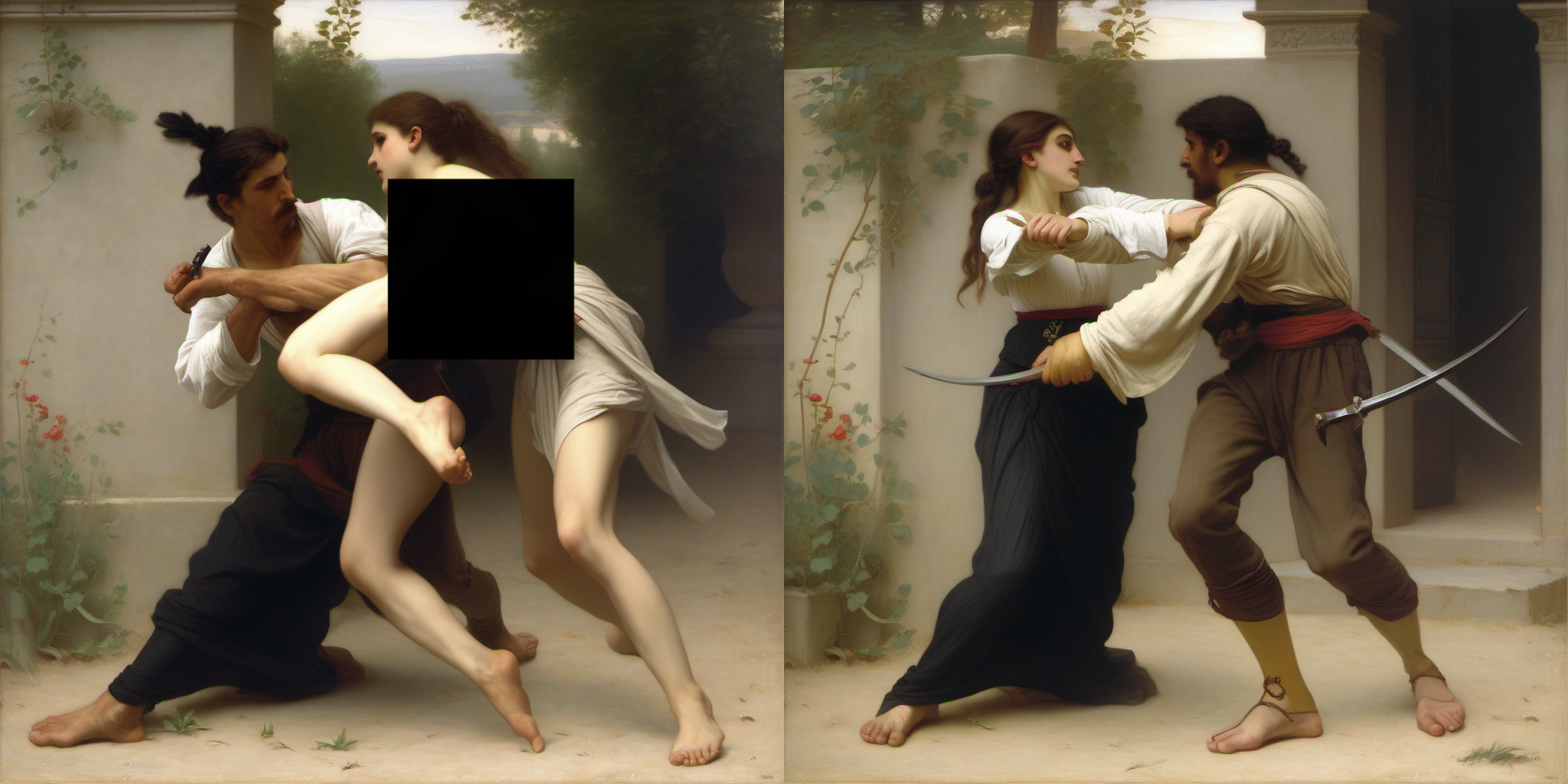}
\subcaption{Comparison for I2P case 29.}
\label{fig:nudity_residual_case}
\end{minipage}
\caption{(a) Feature 6223 responds strongly to horse, cat, and dog, with spatial activation concentrated on the quadruped body. (b) I2P case 29 retains a \texttt{FEET\_EXPOSED} detection after erasure.}
\label{fig:shared_feature_and_residual}
\endgroup
\begin{minipage}[t]{0.55\linewidth}
    \centering
    \includegraphics[width=\linewidth]{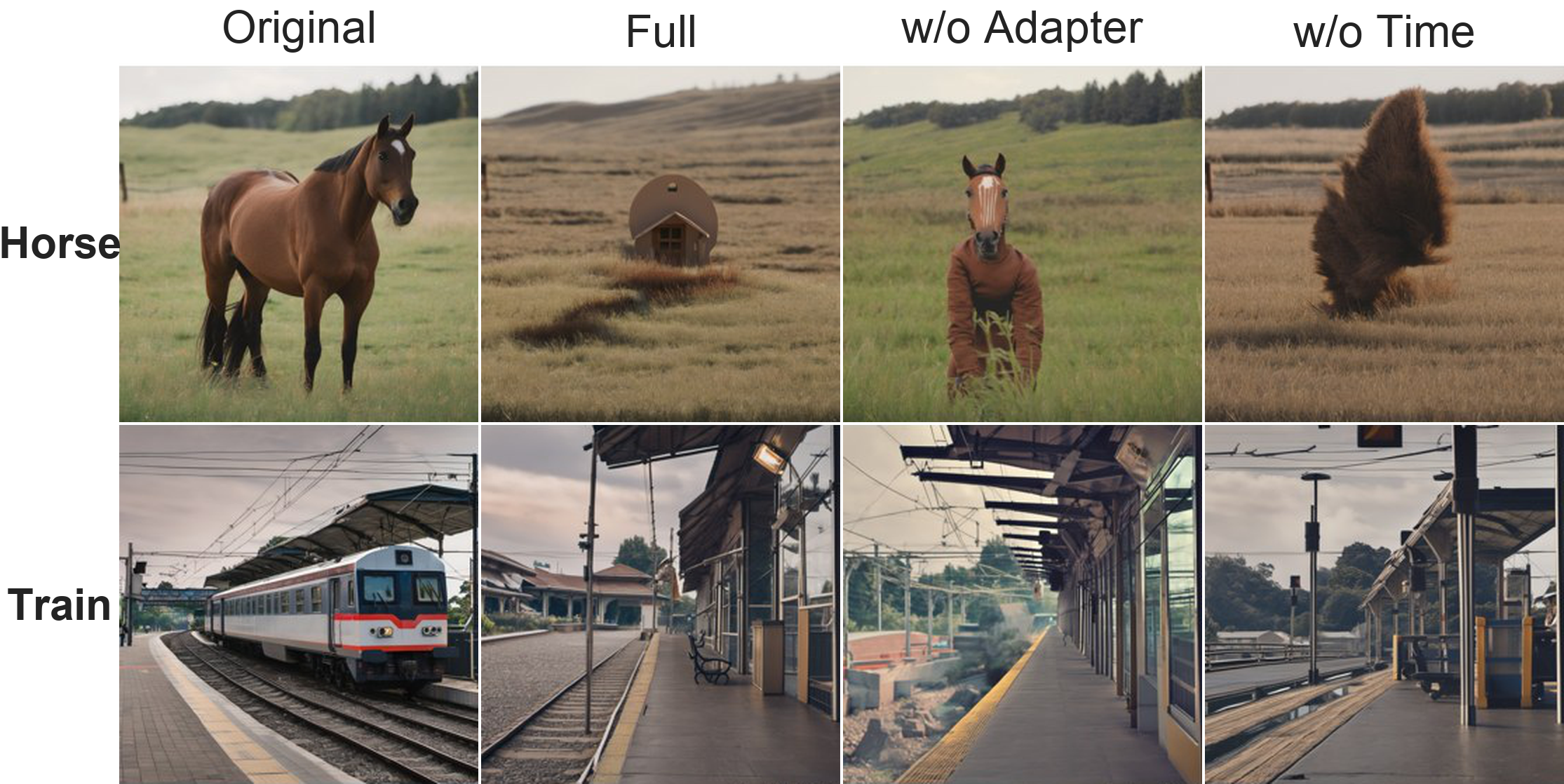}
    \subcaption{Architecture-ablation erasure results.}
    \label{fig:erasure_architecture_ablation}
\end{minipage}
\hfill
\begin{minipage}[t]{0.43\linewidth}
    \centering
    \includegraphics[width=\linewidth]{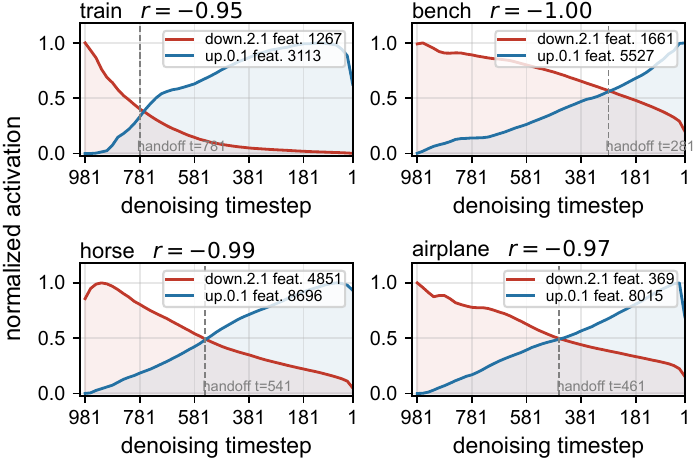}
    \subcaption{Cross-block temporal succession.}
    \label{fig:concept_flow}
\end{minipage}
\caption{(a) Full refers to the complete SharedSAE model. Removing the block adapter can leave the target recognizable, whereas removing temporal conditioning can introduce residual artifacts. (b) Representative \texttt{down.2.1} features decay while \texttt{up.0.1} features rise.}
\label{fig:intervention_handoff}
\end{figure*}

\noindent\textbf{General Concept Erasure.} Figure~\ref{fig:sharedsae_erasure} shows erasure results for 16 concepts. The same intervention suppresses the target for most of these concepts.

\noindent\textbf{Nudity Erasure.} Following the evaluation setting of SPACE \citep{novello2026space}, we evaluate SharedSAE on the I2P-142 nudity subset to align with previous work. SharedSAE substantially reduces NudeNet detections, as shown in Table~\ref{tab:nudity_compare}. Under our detector-positive protocol over the full I2P set, the detector-positive subset contains 139 cases, of which only 1 remains NudeNet-positive after SharedSAE intervention. Figure~\ref{fig:nudity_spatial} shows that the localized features concentrate on exposed body regions, enabling spatially focused suppression. Figure~\ref{fig:shared_feature_and_residual}\subref{fig:nudity_residual_case} presents the sole residual detection, which comes from a visible foot that remains after erasure.

\noindent\textbf{Style Erasure.} Unlike object concepts with identifiable subjects, global styles are encoded by distributed visual attributes such as color, illumination, texture, contour, and scene composition, making them substantially harder to erase. We evaluate five representative styles: cyberpunk, Monet, ukiyo-e, Van Gogh, and watercolor. As shown in Figure~\ref{fig:style_erasure_blacklist}\subref{fig:style_kloc}, increasing the per-block feature set from 10 to 50 generally weakens the target style, but does not fully remove it or consistently restore a neutral photographic appearance. This difficulty is particularly evident for cyberpunk, whose characteristic features, including strong lighting effects, neon colors, reflections, and scene-wide illumination patterns, are distributed across signs, roads, building facades, and the distant streetscape. Consequently, many highly activated features are spatially dispersed and can be readily filtered out by the blacklist. Figure~\ref{fig:style_erasure_blacklist}\subref{fig:style_cyberpunk} illustrates different erasure effects under different blacklist settings.

Figure~\ref{fig:style_erasure_blacklist}\subref{fig:style_localization} shows the erasure results obtained by directly using the prompts employed for concept localization. Every evaluated style is attenuated under this setting. Cyberpunk and Monet are transformed into neutral appearances, the watercolor washes largely disappear, the flat color planes and bold outlines of ukiyo-e weaken visibly, and the characteristic visual patterns of Van Gogh are mostly gone. Together with the observations in Figure~\ref{fig:style_erasure_blacklist}\subref{fig:style_cyberpunk}, these results show that style concepts are erasable and that the main difficulty lies in accurately localizing the relevant features. Style concepts are less concrete than object concepts and more broadly distributed across visual features, so identifying a compact and precise set of features for a target style remains the primary challenge.

\begin{figure*}[t]
    \hfill
    \begin{minipage}[b]{0.25\textwidth}
        \centering
        \includegraphics[width=\linewidth]{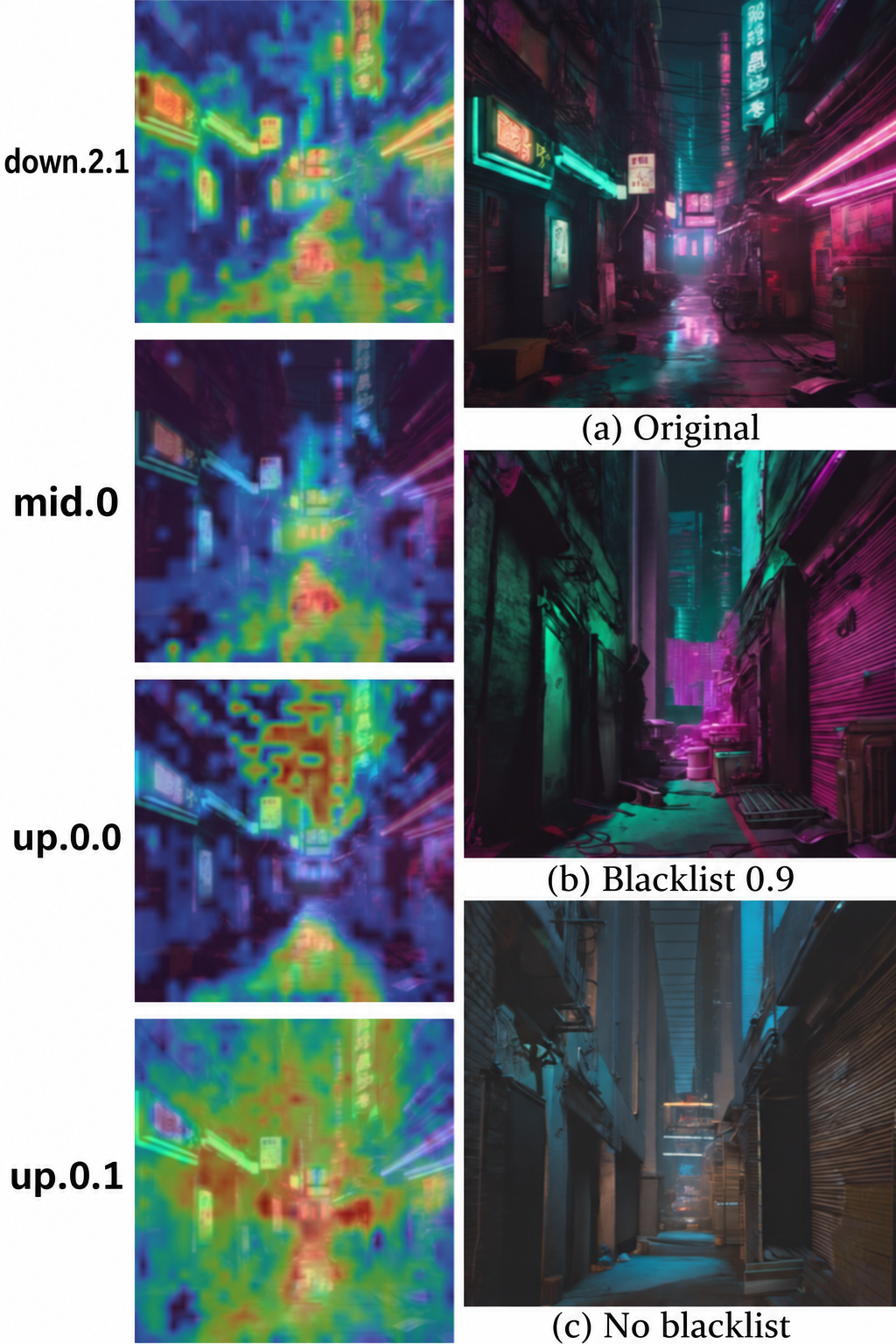}
        \subcaption{Cyberpunk style.}
        \label{fig:style_cyberpunk}
    \end{minipage}
    \hfill
    \begin{minipage}[b]{0.42\textwidth}
        \centering
        \includegraphics[width=\linewidth]{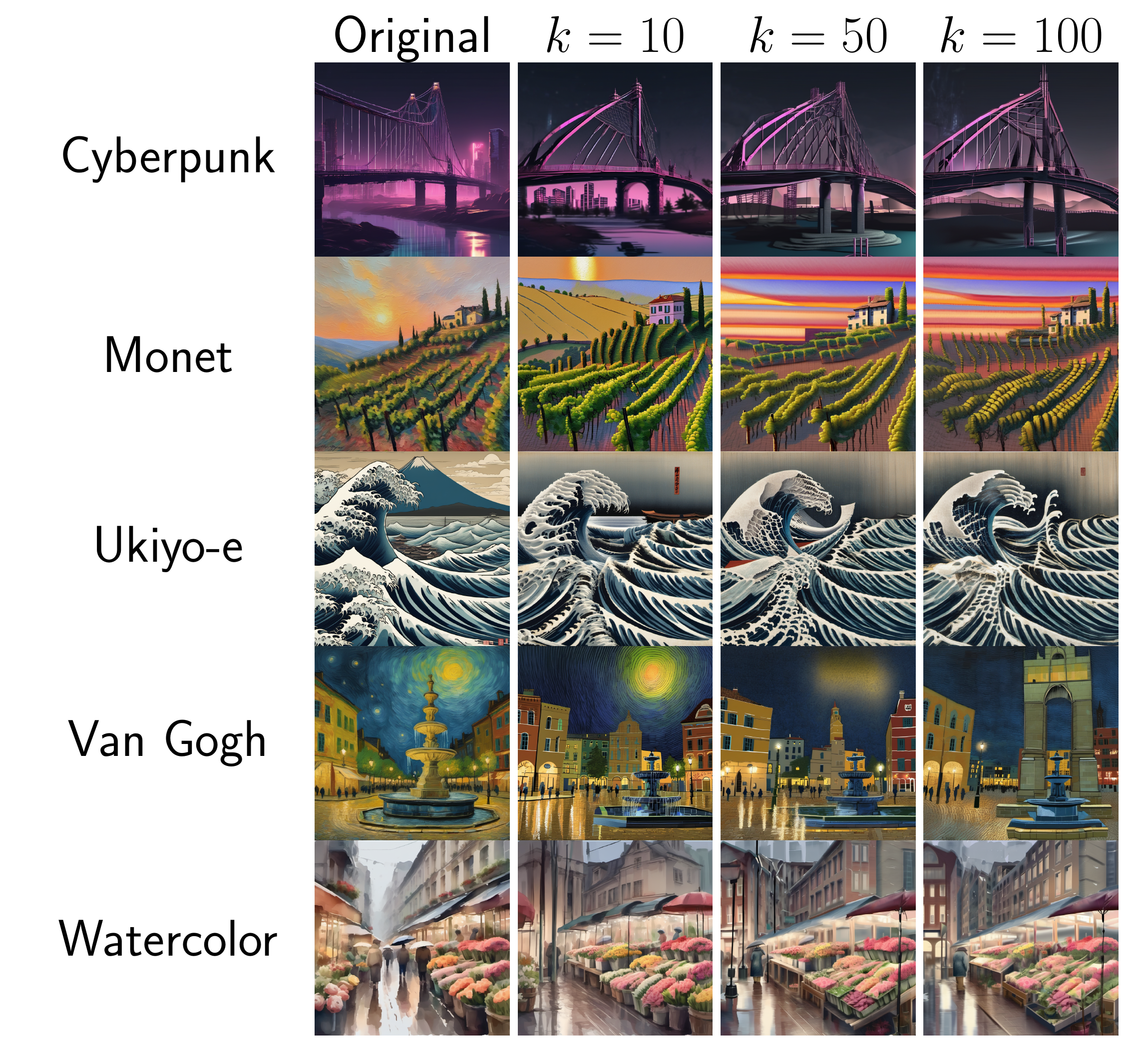}
        \subcaption{Cross-style erasure as $k_{\mathrm{loc}}$ increases.}
        \label{fig:style_kloc}
    \end{minipage}
    \hfill
    \begin{minipage}[b]{0.27\textwidth}
        \centering
        \includegraphics[width=\linewidth]{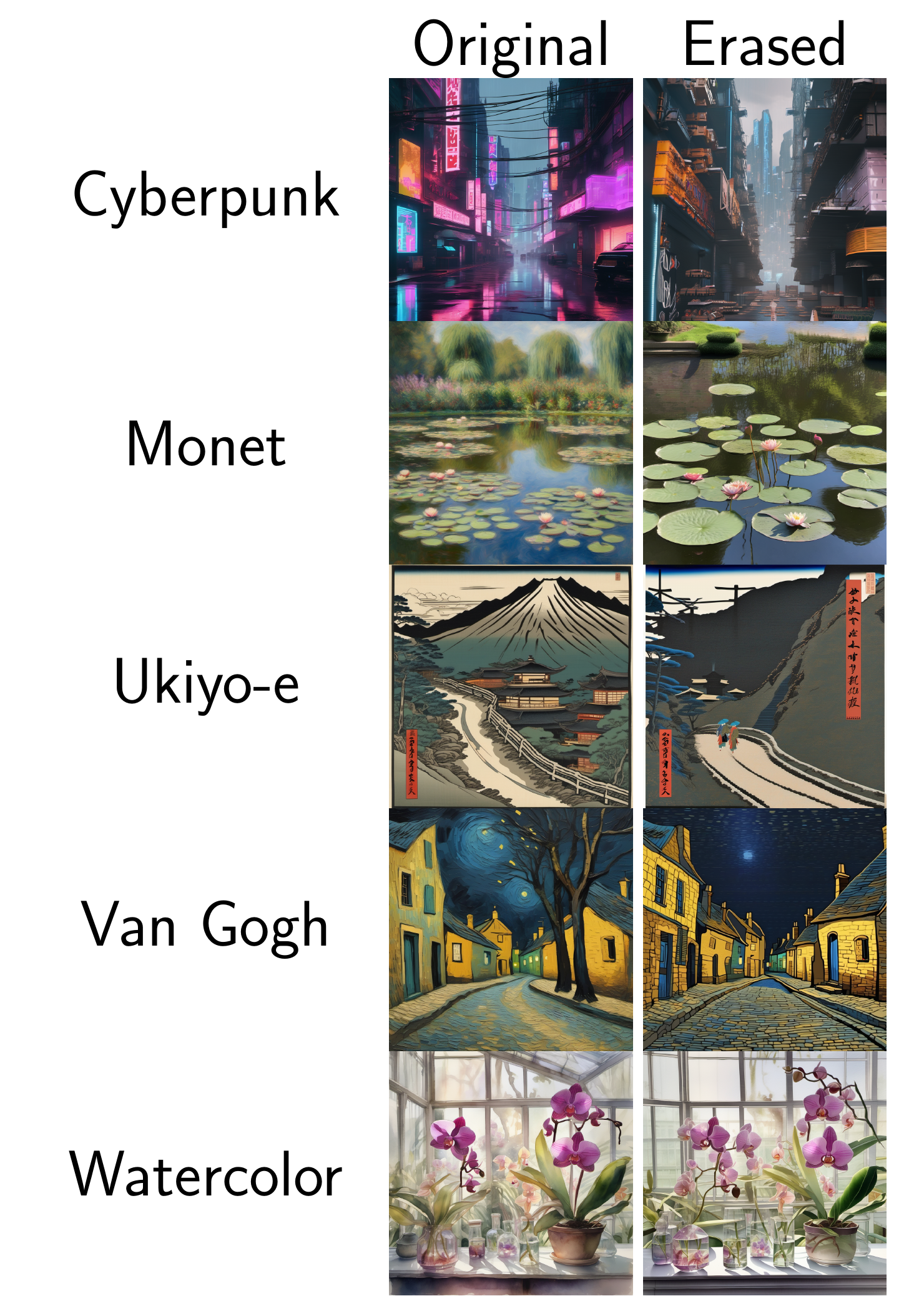}
        \subcaption{Successful examples.}
        \label{fig:style_localization}
    \end{minipage}
    \hfill
    \caption{Style erasure. (a) Cyberpunk block-wise heatmaps accompany the original, blacklist-0.9, and no-blacklist outputs. (b) Increasing the number of features used per block, $k_{\mathrm{loc}}$, from 10 to 50 attenuates each style without removing it. (c) The prompts used for concept localization erase the target style.}
    \label{fig:style_erasure_blacklist}
\end{figure*}

\noindent\textbf{Concept Feature Sharing.} Figure~\ref{fig:shared_feature_and_residual}\subref{fig:shared_feature} illustrates how a single SharedSAE feature can recur across semantically related concepts. Among the features localized at \texttt{mid.0}, feature 6223 ranks in the top-$k_{\mathrm{loc}}$ for cat, dog, and horse, while its mean activation on the remaining concepts is one to two orders of magnitude lower. Its spatial activation consistently follows the animal body despite changes in species, pose, and appearance, indicating that the feature captures visual structure shared across quadruped concepts rather than a single object class. Suppressing this shared feature can therefore influence several related concepts, linking semantic reuse in the dictionary to cross-concept coupling during intervention.

This feature direction sharing also manifests as concept compensation and leakage during erasure. Suppressing features associated with a horse or motorcycle can redirect the representation toward a nearby concept direction, such as a person. Because the coarse spatial structure of the target is established early in denoising and tends to persist throughout the trajectory, the replacement often inherits the original shape or spatial layout, even when it belongs to a different semantic category. Moreover, erasing a green truck removes the truck while retaining the green attribute. Green leaks into the background image and is automatically added to its original position. This suggests that object identity and attribute are partially disentangled in a shared sparse direction.

\noindent\textbf{Spatiotemporal Feature Succession.} Figure~\ref{fig:concept_dynamics}\subref{fig:feature_succession} shows that concept-representative feature sets change throughout denoising. The mean top-5 overlap in \texttt{down.2.1} decreases from 0.746 to 0.485, with the same downward trend across all four blocks. Raw activation magnitudes also carry a strong temporal trend, and target-feature activations in \texttt{down.2.1} consistently decay from early to late timesteps. We therefore use relative block shares to factor out global magnitude changes and isolate redistribution across blocks. Figure~\ref{fig:concept_dynamics}\subref{fig:layer_share} shows that although \texttt{down.2.1} remains the dominant contributor for most concepts, its relative share generally decreases over time, while \texttt{mid.0} and \texttt{up.0.1} make increasingly larger contributions. 

This pattern suggests that concept representations are progressively redistributed from earlier downsampling blocks toward intermediate and later stages of the U-Net as denoising proceeds. This succession is also visible at the level of individual features. Figure~\ref{fig:intervention_handoff}\subref{fig:concept_flow} shows highly ranked \texttt{down.2.1} and \texttt{up.0.1} features for train, bench, horse, and airplane. The \texttt{down.2.1} features activate early and then decay, whereas the \texttt{up.0.1} features emerge later. Their trajectories cross during denoising, showing the succession of representative features across blocks.
\begin{figure*}[!t]
    \centering
    \begin{minipage}[t]{0.49\textwidth}
        \centering
        \includegraphics[width=\linewidth]{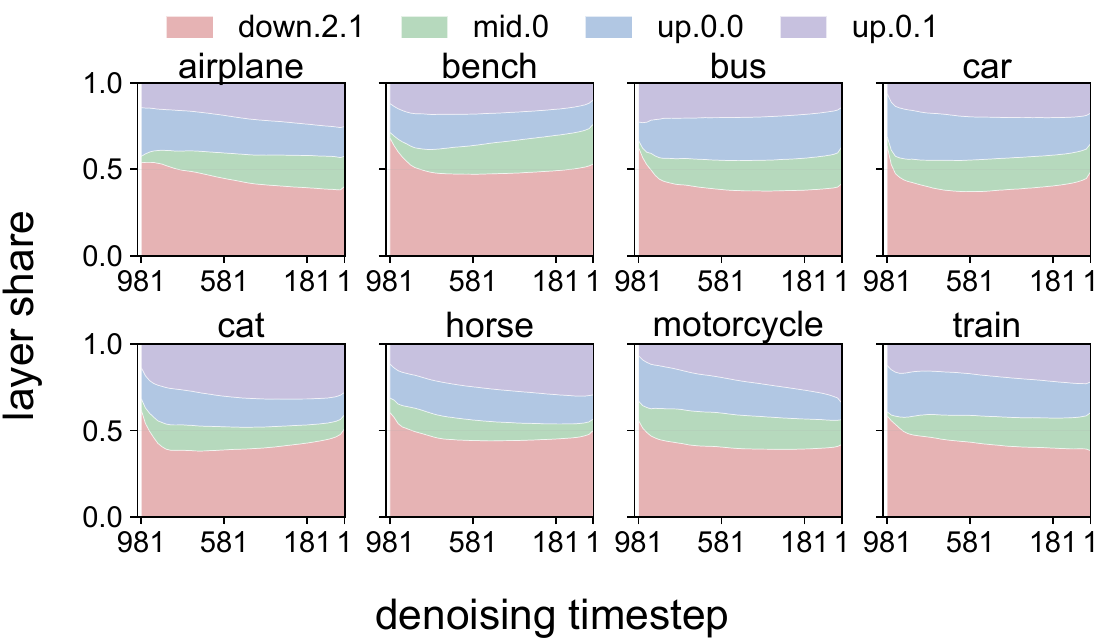}
        \subcaption{Block-wise shares of concept activation.}
        \label{fig:layer_share}
    \end{minipage}
    \hfill
    \begin{minipage}[t]{0.49\textwidth}
        \centering
        \includegraphics[width=\linewidth]{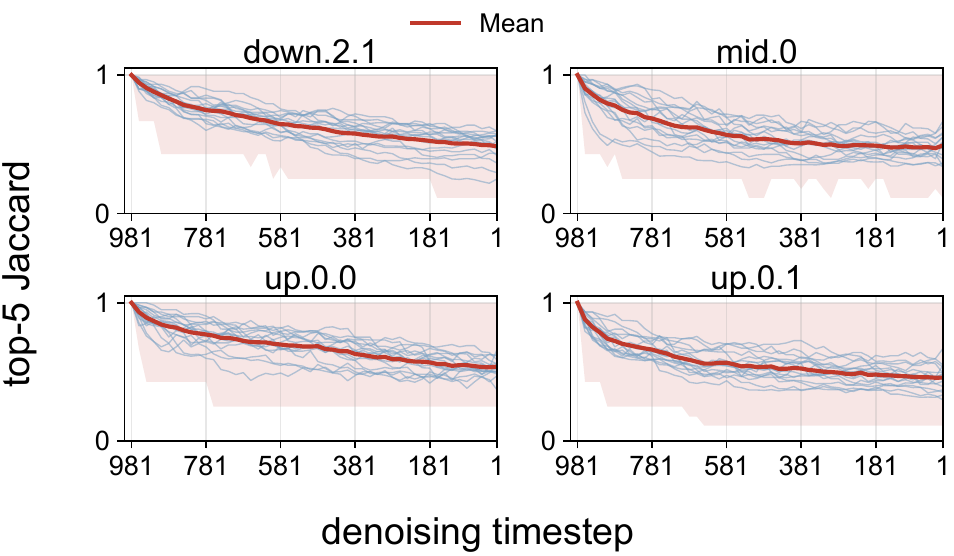}
        \subcaption{Feature-set succession within each block.}
        \label{fig:feature_succession}
    \end{minipage}
    \caption{Spatiotemporal evolution of concept-representative features. (a) Relative activation shares generally shift from \texttt{down.2.1} toward \texttt{mid.0} or \texttt{up.0.1}. (b) Top-5 feature overlap with the initial timestep decreases throughout denoising; thick lines and shaded regions denote the mean and $95\%$ interval over 340 cases.}
    \label{fig:concept_dynamics}
\end{figure*}

\section{Conclusion}
SharedSAE shows that sparse analysis of diffusion models need not be fragmented across network depth and denoising time. It learns a single dictionary shared across four U-Net blocks and the full denoising trajectory, supporting reconstruction, localization, and intervention throughout generation. This shared space is not only more compact than maintaining separate block-wise SAEs, but also makes it possible to follow how semantic representations evolve: concept-related features are reused across related concepts and progressively redistributed from early downsampling blocks toward intermediate and later blocks as denoising proceeds. The same representation supports effective intervention, including a reduction from 139 NudeNet-positive generations to 1 in our detector-positive I2P subset. Taken together, our results suggest that a shared sparse vocabulary can provide a fixed reference frame for studying representations that themselves evolve across network depth and denoising time.

\newpage
\section*{AI Use Statement}
In this work, we did not use generative AI tools for any task requiring disclosure under the ICLR AI policy. Specifically, generative AI was not used to create synthetic datasets; develop theoretical models or conceptual frameworks; derive mathematical statements; supply key components of or aid in constructing mathematical proofs; formulate or refine research hypotheses; design, select, or evaluate experimental methodology; implement our proposed algorithm; translate research materials; clean or restructure datasets; conduct qualitative or thematic analysis; or interpret experimental findings.

Generative AI tools were used for tasks for which disclosure is recommended: drafting and revising manuscript prose to improve readability, suggesting structural improvements to the paper, editing source code, and polishing the presentation and LaTeX typesetting of mathematical derivations. Generative AI did not generate the core mathematical claims or proofs.

All AI-assisted outputs were reviewed by the authors. AI-modified code was inspected and tested for correctness. AI-edited manuscript text was cross-checked against our original experimental results, mathematical derivations, and intended claims, to ensure no unsubstantiated assertions were introduced and that the interpretation of our findings remained unchanged. The authors take full responsibility for all final content of this work, including text, claims, and artifacts produced with generative AI assistance.

\bibliography{sharedsae_references}
\bibliographystyle{iclr2027_conference}


\setlength{\abovedisplayskip}{10pt plus 2pt minus 5pt}
\setlength{\belowdisplayskip}{10pt plus 2pt minus 5pt}
\setlength{\abovedisplayshortskip}{0pt plus 3pt}
\setlength{\belowdisplayshortskip}{6pt plus 3pt minus 3pt}
\appendix

\section{Sparse Vector-Field Interpretation of SharedSAE}
\label{sec:analytical_properties}

SharedSAE acts on the block residuals $x=h^{\mathrm{out}}-h^{\mathrm{in}}$ of the epsilon-predicting SDXL U-Net.
Although SDXL does not directly regress an ODE velocity, an epsilon prediction induces a score estimate,
and hence a probability-flow ODE associated with the same diffusion probability path.
This section uses that standard correspondence only as a bridge from hidden-state perturbations to
a local generative vector field. The SharedSAE-specific statement is then that a decoder direction
is pushed forward through the downstream U-Net Jacobian into a local velocity direction, and that
the final image-space effect is the accumulated transport of these local perturbations along the
denoising trajectory.

\subsection{From Epsilon Prediction to a Probability-Flow Velocity}

Let the Gaussian diffusion path be
\begin{equation}
y_t = a_t y_0 + s_t \epsilon,
\qquad
\epsilon \sim \mathcal{N}(0,I),
\label{eq:gaussian_diffusion_path}
\end{equation}
where $a_t>0$ and $s_t>0$ are differentiable signal and noise schedules.
A continuous-time linear SDE that realizes these Gaussian conditionals can be written as
\begin{equation}
\mathrm{d}y
=
f_t y\,\mathrm{d}t
+
g_t\,\mathrm{d}W_t,
\qquad
f_t = \frac{\dot a_t}{a_t},
\label{eq:linear_forward_sde}
\end{equation}
with
\begin{equation}
g_t^2
=
\frac{\mathrm{d}}{\mathrm{d}t}s_t^2
-
2 f_t s_t^2.
\label{eq:diffusion_coefficient_from_schedule}
\end{equation}
Under the usual epsilon-prediction parameterization,
the exact conditional score satisfies
\begin{equation}
\nabla_y \log p_t(y_t\mid c)
=
-\frac{\mathbb{E}[\epsilon\mid y_t,c]}{s_t}.
\label{eq:exact_score_epsilon}
\end{equation}
A trained denoiser approximates the conditional noise expectation, so we define
\begin{equation}
s_\theta(y_t,t,c)
:=
-\frac{\epsilon_\theta(y_t,t,c)}{s_t}
\approx
\nabla_y \log p_t(y_t\mid c).
\label{eq:epsilon_to_score}
\end{equation}

The probability-flow ODE associated with the forward SDE is
\begin{equation}
\frac{\mathrm{d}y}{\mathrm{d}t}
=
v_{\mathrm{pf}}(y,t,c)
=
f_t y
-
\frac{1}{2}g_t^2 s_\theta(y,t,c).
\label{eq:probability_flow_ode}
\end{equation}
Substituting Eq.~(\ref{eq:epsilon_to_score}) gives
\begin{equation}
v_{\mathrm{pf}}(y,t,c)
=
f_t y
+
\kappa_t \epsilon_\theta(y,t,c),
\qquad
\kappa_t
=
\frac{g_t^2}{2s_t}.
\label{eq:epsilon_to_velocity}
\end{equation}
Equivalently, using Eq.~(\ref{eq:diffusion_coefficient_from_schedule}),
\begin{equation}
\kappa_t
=
\dot s_t
-
\frac{\dot a_t}{a_t}s_t.
\label{eq:kappa_schedule_form}
\end{equation}

\paragraph{Effective prediction under classifier-free guidance.}
The experiments use the prediction actually consumed by the sampler.
To cover both conditional sampling and classifier-free guidance (CFG), let
\begin{equation}
\widehat{\epsilon}_\theta(y_t,t,c)
\label{eq:effective_epsilon}
\end{equation}
denote the \emph{effective} noise prediction passed to the scheduler.
Without CFG, $\widehat{\epsilon}_\theta=\epsilon_\theta(\cdot\mid c)$.
With guidance scale $w$,
\begin{equation}
\widehat{\epsilon}_\theta
=
\epsilon_\theta(\cdot\mid \varnothing)
+
w\left[
\epsilon_\theta(\cdot\mid c)
-
\epsilon_\theta(\cdot\mid \varnothing)
\right].
\label{eq:cfg_effective_epsilon}
\end{equation}
For guided sampling, Eq.~(\ref{eq:epsilon_to_velocity}) should therefore be interpreted as the effective denoising vector field used by the sampler, rather than as a claim that CFG exactly equals the probability-flow ODE of the original conditional density. The same derivation, however,
applies whether the intervention modifies the conditional branch, the unconditional branch, or both.

\paragraph{Standard marginal-equivalence fact.}
If $s_\theta=\nabla_y\log p_t$ is exact, then the probability-flow ODE in
Eq.~(\ref{eq:probability_flow_ode}) and the corresponding reverse-time SDE share the same
one-time marginals $p_t$ \citep{song2021score}.
We use this standard result only to justify the velocity-field interpretation; the propositions below
concern the SharedSAE intervention itself.

\subsection{SharedSAE Features as Pushed-Forward Velocity Directions}

For block $b$ and denoising time $t$, let $x_{b,t}$ be the residual update.
Following Eq.~(2) in the main text, we define the per-coordinate residual scale
\begin{equation}
\sigma_b
=
\frac{\max(\mu_b,\epsilon)}{\sqrt{d}},
\qquad
\widetilde{x}_{b,t}
=
\frac{x_{b,t}}{\sigma_b}.
\label{eq:normalization_convention}
\end{equation}
Therefore, a vector represented in the normalized SharedSAE residual space is mapped back to the
original U-Net residual space by multiplication with $\sigma_b$.

Consistent with Eq.~(4) in the main text, SharedSAE reconstructs the normalized residual as
\begin{equation}
\widetilde{x}_{b,t}
\approx
Dz_{b,t}+\beta_x,
\label{eq:sae_normalized_reconstruction}
\end{equation}
and hence
\begin{equation}
x_{b,t}
\approx
\sigma_b
\left(
Dz_{b,t}+\beta_x
\right)
=
\sigma_b
\sum_{j=1}^{m}z_{j,b,t}D_j
+
\sigma_b\beta_x.
\label{eq:sae_residual_decomposition}
\end{equation}

As in the main text, let $h^{\mathrm{out}}_{b,t}$ denote the block output activation.
For the Jacobian calculation we regard $h^{\mathrm{out}}_{b,t}$ as the vectorized block state and
suppress the spatial-token index $i$ for readability; equivalently, each $D_j$ is injected at the
corresponding token coordinates.
Holding the upstream state fixed locally, we write the downstream map from the block output to the effective sampler prediction as
\begin{equation}
\widehat{\epsilon}_\theta(y_t,t,c)
=
G_{b,t,c}\!\left(h^{\mathrm{out}}_{b,t}\right),
\label{eq:downstream_unet_map}
\end{equation}
with Jacobian
\begin{equation}
J_{b,t,c}
=
\frac{
\partial G_{b,t,c}(h^{\mathrm{out}}_{b,t})
}{
\partial h^{\mathrm{out}}_{b,t}
}.
\label{eq:downstream_jacobian}
\end{equation}

\begin{proposition}[Sparse velocity pushforward]
\label{prop:sparse_velocity_pushforward}
Assume $G_{b,t,c}$ is twice continuously differentiable in a neighborhood of the current state.
For a sufficiently small sparse-code perturbation $\Delta z$ applied at block $b$ and time $t$,
the induced change in the effective probability-flow velocity satisfies
\begin{equation}
\Delta v_{\mathrm{pf}}
=
\sum_{j=1}^{m}
\Delta z_j
\psi^{\mathrm{ODE}}_{j,b,t,c}
+
\mathcal{O}(\lVert \Delta z\rVert_2^2),
\label{eq:sparse_velocity_basis}
\end{equation}
where
\begin{equation}
\psi^{\mathrm{ODE}}_{j,b,t,c}
=
\kappa_t
\sigma_b
J_{b,t,c}D_j.
\label{eq:ode_feature_direction}
\end{equation}
For the reverse-time SDE, the \emph{feature-induced drift perturbation} is collinear with the ODE
perturbation and has twice its score-dependent coefficient:
\begin{equation}
\Delta b_{\mathrm{SDE}}
=
2\,\Delta v_{\mathrm{pf}}
+
\mathcal{O}(\lVert \Delta z\rVert_2^2).
\label{eq:sde_ode_perturbation_relation}
\end{equation}
This statement concerns the perturbations; the full reverse-SDE drift and full probability-flow
velocity need not themselves be collinear because both also contain the forward drift term.
\end{proposition}

\begin{proof}
A sparse-code perturbation produces a hidden-state displacement
\begin{equation}
\Delta h^{\mathrm{out}}_{b,t}
=
\sigma_b D\Delta z.
\label{eq:hidden_displacement_from_code}
\end{equation}
A first-order Taylor expansion of the downstream U-Net gives
\begin{equation}
\Delta\widehat{\epsilon}_\theta
=
J_{b,t,c}\Delta h^{\mathrm{out}}_{b,t}
+
\mathcal{O}(\lVert\Delta h^{\mathrm{out}}_{b,t}\rVert_2^2).
\label{eq:epsilon_taylor}
\end{equation}
For a finite hidden-state displacement, the fundamental theorem of calculus gives the exact identity
\begin{equation}
G_{b,t,c}(h+\Delta h)-G_{b,t,c}(h)
=
\bar J_{b,t,c}(h,\Delta h)\Delta h,
\label{eq:finite_displacement_jacobian}
\end{equation}
where
\begin{equation}
\bar J_{b,t,c}(h,\Delta h)
=
\int_0^1 J_{G_{b,t,c}}(h+s\Delta h)\,\mathrm{d}s.
\label{eq:path_averaged_jacobian}
\end{equation}
Thus Eq.~(\ref{eq:epsilon_taylor}) replaces the path-averaged Jacobian by its value at the
unperturbed state. Its accuracy is determined by the Jacobian variation along the displacement path.
By Eq.~(\ref{eq:epsilon_to_velocity}), perturbations of the effective noise prediction induce
\begin{equation}
\Delta v_{\mathrm{pf}}
=
\kappa_t
\Delta\widehat{\epsilon}_\theta.
\label{eq:velocity_from_epsilon_perturbation}
\end{equation}
Substituting Eqs.~(\ref{eq:hidden_displacement_from_code}) and
(\ref{eq:epsilon_taylor}) yields
\begin{equation}
\Delta v_{\mathrm{pf}}
=
\kappa_t\sigma_b
J_{b,t,c}D\Delta z
+
\mathcal{O}(\lVert\Delta z\rVert_2^2),
\end{equation}
which gives Eqs.~(\ref{eq:sparse_velocity_basis})--(\ref{eq:ode_feature_direction})
after expanding over decoder columns.

For the reverse-time SDE, the score term has coefficient $g_t^2$ rather than $g_t^2/2$.
Therefore, its perturbation induced by the same score change has twice the coefficient of the
probability-flow ODE perturbation, proving Eq.~(\ref{eq:sde_ode_perturbation_relation}).
\end{proof}

\paragraph{Feature identity versus feature effect.}
The shared decoder column $D_j$ supplies a common sparse coordinate across blocks and timesteps,
whereas its local generative effect is
\begin{equation}
\psi^{\mathrm{ODE}}_{j,b,t,c}
=
\kappa_t\sigma_bJ_{b,t,c}D_j,
\end{equation}
which remains context dependent through the downstream Jacobian and the diffusion schedule.
Thus SharedSAE shares \emph{feature coordinates}, not a single context-invariant velocity vector.

\subsection{Concept Erasure as a Sparse Velocity Perturbation}

Let $\mathcal{F}_{c,b}$ be the set of features localized for concept $c$ at block $b$, and let
\begin{equation}
q_j(b,t)=g_jw_j(b,t)
\end{equation}
combine the rank and temporal weights.
The main text denotes the global erasure strength by $\lambda_{\mathrm{erase}}$ in Eq.~(9), with
$\lambda_{\mathrm{erase}}=10$ for general concept erasure and $\lambda_{\mathrm{erase}}=5$ for I2P
nudity erasure. We write $\lambda$ for this fixed erasure strength below. The adapter scale
$\alpha_A$ from Eq.~(3) is a separate parameter and does not enter the analysis.
The normalized-space concept direction is
\begin{equation}
r_c(b,t)
=
\sum_{j\in\mathcal{F}_{c,b}}
q_j(b,t)\,z_{j,b,t}D_j.
\label{eq:normalized_concept_direction}
\end{equation}
Consistent with Eq.~(\ref{eq:normalization_convention}) and the main-text intervention in Eq.~(9),
the corresponding U-Net-space displacement is
\begin{equation}
\Delta h^{\mathrm{out}}_{b,t}
=
-\lambda
\sigma_b
r_c(b,t).
\label{eq:hidden_concept_subtraction}
\end{equation}
The resulting first-order velocity change is therefore
\begin{equation}
\Delta v_{\mathrm{pf}}
\approx
-\lambda
\sum_{j\in\mathcal{F}_{c,b}}
q_j(b,t)z_{j,b,t}
\psi^{\mathrm{ODE}}_{j,b,t,c}.
\label{eq:concept_velocity_subtraction}
\end{equation}

The implementation additionally caps the mean absolute per-step perturbation ratio at
$\delta_{\max}=0.2$, as stated in the main text.
The local analysis above describes the unclipped branch of the intervention; when the cap is active,
the same pushforward applies to the clipped displacement away from the clipping boundary.

We define the \emph{localized feature-induced velocity subspace}
\begin{equation}
\mathcal{V}_c(b,t)
=
\operatorname{span}
\left\{
\psi^{\mathrm{ODE}}_{j,b,t,c}
\;\middle|\;
j\in\mathcal{F}_{c,b}
\right\}.
\label{eq:localized_velocity_subspace}
\end{equation}
Equation~(\ref{eq:concept_velocity_subtraction}) shows that the intervention subtracts a
state-dependent vector in this empirically localized subspace.
This does not assert that $\mathcal{V}_c$ is the unique or complete semantic subspace of concept $c$;
rather, it is the subspace induced by the features selected by the localization procedure.

\paragraph{Multiple intervened blocks.}
To separate the fixed erasure strength from the expansion parameter, define the complete block-wise
displacement $u_b=-\lambda\sigma_b r_c(b,t)$ and introduce the auxiliary path
\begin{equation}
\Delta h_b(\varepsilon)=\varepsilon u_b,
\qquad
\varepsilon\in[0,1].
\label{eq:auxiliary_intervention_path}
\end{equation}
Here $\varepsilon=0$ is the unperturbed model, and $\varepsilon=1$ corresponds to the complete intervention
used in the experiments. A multivariate Taylor expansion at $\varepsilon=0$ gives
\begin{equation}
\Delta\widehat{\epsilon}_\theta(\varepsilon)
=
\varepsilon
\sum_{b\in\mathcal{B}}J_{b,t,c}u_b
+
\mathcal{O}(\varepsilon^2).
\label{eq:multiblock_additivity}
\end{equation}
For two successive blocks, an upstream intervention changes the later sparse code according to
$z_2(\varepsilon)=z_2(0)+\varepsilon\dot z_2(0)+\mathcal{O}(\varepsilon^2)$.
The later intervention contains another factor of $\varepsilon$, so the induced cross-block term is
$\mathcal{O}(\varepsilon^2)$, not $\mathcal{O}(\lambda^2)$. Equation~(\ref{eq:multiblock_additivity})
is local at $\varepsilon=0$. At $\varepsilon=1$, its accuracy is determined by nonlinear
interactions along the complete intervention path.

\paragraph{Local magnitude bound.}
Let $D_{\mathcal{F}_{c,b}}$ denote the decoder submatrix containing the localized columns and let
$a_{b,t}$ be the coefficient vector with entries
$a_j=q_j(b,t)z_{j,b,t}$.
Then
\begin{equation}
\lVert\Delta v_{\mathrm{pf}}\rVert_2
\lesssim
\lambda
|\kappa_t|\sigma_b
\lVert J_{b,t,c}\rVert_2
\lVert D_{\mathcal{F}_{c,b}}\rVert_2
\lVert a_{b,t}\rVert_2.
\label{eq:velocity_magnitude_bound}
\end{equation}
Since each decoder column is unit norm, a looser but simpler bound is
\begin{equation}
\lVert\Delta v_{\mathrm{pf}}\rVert_2
\lesssim
\lambda
|\kappa_t|\sigma_b
\lVert J_{b,t,c}\rVert_2
\sum_{j\in\mathcal{F}_{c,b}}
|q_j(b,t)z_{j,b,t}|.
\label{eq:velocity_magnitude_bound_l1}
\end{equation}

Because both the residual and the intervention are mapped from normalized space by the same residual scale $\sigma_b$, block normalization preserves the relative intervention scale before the downstream Jacobian acts:
\begin{equation}
\frac{
\lVert\Delta h^{\mathrm{out}}_{b,t}\rVert_2
}{
\lVert x_{b,t}\rVert_2
}
=
\lambda
\frac{
\left\lVert
\sum_{j\in\mathcal{F}_{c,b}}
q_j(b,t)z_{j,b,t}D_j
\right\rVert_2
}{
\lVert\widetilde{x}_{b,t}\rVert_2
}.
\label{eq:relative_hidden_intervention}
\end{equation}
The residual scale $\sigma_b$ cancels, so the same normalized intervention has a comparable
relative meaning across blocks before block-specific downstream dynamics are applied.

\subsection{Accumulation Along the Denoising Trajectory}

The continuous-time analysis below is an idealization of the discrete sampler used in the experiments.
Let $\tau\in[0,1]$ increase from noise to data and absorb the reversal of the original diffusion time
into an effective vector field $\bar v$:
\begin{equation}
\frac{\mathrm{d}y}{\mathrm{d}\tau}
=
\bar v(y,\tau,c),
\qquad
y(0)=y_{\mathrm{noise}}.
\label{eq:reverse_sampling_clock}
\end{equation}
Write the intervened dynamics as
\begin{equation}
\frac{\mathrm{d}y^\varepsilon}{\mathrm{d}\tau}
=
\bar v(y^\varepsilon,\tau,c)
+
\varepsilon
u(y^\varepsilon,\tau,c),
\label{eq:perturbed_continuous_dynamics}
\end{equation}
where $u$ collects the complete pushed-forward contributions at the fixed strength $\lambda$.

Define
\begin{equation}
A(\tau)
=
J_y\bar v(y_\tau,\tau,c)
\end{equation}
and the state-transition matrix $\Phi(\tau,s)$ by
\begin{equation}
\frac{\partial \Phi(\tau,s)}{\partial\tau}
=
A(\tau)\Phi(\tau,s),
\qquad
\Phi(s,s)=I.
\label{eq:state_transition_matrix}
\end{equation}

\begin{theorem}[First-order trajectory displacement]
\label{thm:trajectory_variation}
Assume $\bar v$ and $u$ are twice continuously differentiable in a neighborhood of the unperturbed
trajectory and that the solution remains in this neighborhood for sufficiently small
$\varepsilon$.
Then
\begin{equation}
y^\varepsilon(1)-y(1)
=
\varepsilon
\int_0^1
\Phi(1,\tau)
u(y_\tau,\tau,c)
\,\mathrm{d}\tau
+
\mathcal{O}(\varepsilon^2).
\label{eq:trajectory_variation}
\end{equation}
For SharedSAE,
\begin{equation}
u(y_\tau,\tau,c)
=
-\lambda
\sum_{b\in\mathcal{B}}
\sum_{j\in\mathcal{F}_{c,b}}
q_j(b,t(\tau))
z_{j,b,t(\tau)}
\bar\psi_{j,b,t(\tau),c},
\label{eq:sharedsae_continuous_u}
\end{equation}
where $\bar\psi$ denotes the local pushed-forward direction after the change of clock from $t$ to $\tau$.
Hence
\begin{equation}
\begin{aligned}
y^\varepsilon(1)-y(1)
\approx
-\varepsilon\lambda
\sum_{b\in\mathcal{B}}
\sum_{j\in\mathcal{F}_{c,b}}
\int_0^1
&\Phi(1,\tau)\,
q_j(b,t(\tau))\,
z_{j,b,t(\tau)}
\\
&\cdot
\bar\psi_{j,b,t(\tau),c}
\,\mathrm{d}\tau.
\end{aligned}
\label{eq:sharedsae_trajectory_integral}
\end{equation}
\end{theorem}

\begin{proof}
Let
\begin{equation}
\xi(\tau)
=
\left.
\frac{\partial y^\varepsilon(\tau)}
{\partial \varepsilon}
\right|_{\varepsilon=0}.
\end{equation}
Differentiating Eq.~(\ref{eq:perturbed_continuous_dynamics}) with respect to
$\varepsilon$ at zero gives the first-variation equation
\begin{equation}
\frac{\mathrm{d}\xi}{\mathrm{d}\tau}
=
A(\tau)\xi
+
u(y_\tau,\tau,c),
\qquad
\xi(0)=0.
\label{eq:first_variation_equation}
\end{equation}
Variation of constants yields
\begin{equation}
\xi(1)
=
\int_0^1
\Phi(1,\tau)
u(y_\tau,\tau,c)
\,\mathrm{d}\tau.
\label{eq:first_variation_solution}
\end{equation}
A second-order Taylor expansion in $\varepsilon$ then gives
\begin{equation}
y^\varepsilon(1)-y(1)
=
\varepsilon\xi(1)
+
\mathcal{O}(\varepsilon^2),
\end{equation}
proving Eq.~(\ref{eq:trajectory_variation}).
Substituting the sum of SharedSAE pushed-forward feature contributions from
Eq.~(\ref{eq:concept_velocity_subtraction}) and
Eq.~(\ref{eq:multiblock_additivity}) gives
Eq.~(\ref{eq:sharedsae_trajectory_integral}).
\end{proof}

\subsection{Discrete-Sampler Counterpart}

The experiments use a finite-step sampler, so it is useful to state the discrete analogue explicitly.
Let one unperturbed sampling step be
\begin{equation}
y_{k+1}
=
F_k(y_k),
\qquad
k=0,\ldots,N-1,
\label{eq:discrete_sampler}
\end{equation}
and let $U_k(y_k)$ be the complete SharedSAE intervention at step $k$, including the fixed strength
$\lambda$:
\begin{equation}
y^\varepsilon_{k+1}
=
F_k(y^\varepsilon_k)
+
\varepsilon U_k(y^\varepsilon_k)
+
\mathcal{O}(\varepsilon^2).
\label{eq:discrete_perturbed_sampler}
\end{equation}
Define the discrete transition matrix
\begin{equation}
\Phi_{N,k+1}
=
J F_{N-1}(y_{N-1})
\cdots
J F_{k+1}(y_{k+1}),
\label{eq:discrete_state_transition}
\end{equation}
with $\Phi_{N,N}=I$.

\begin{corollary}[First-order displacement for a discrete sampler]
\label{cor:discrete_trajectory_variation}
Under the same local smoothness assumptions,
\begin{equation}
y^\varepsilon_N-y_N
=
\varepsilon
\sum_{k=0}^{N-1}
\Phi_{N,k+1}
U_k(y_k)
+
\mathcal{O}(\varepsilon^2).
\label{eq:discrete_trajectory_variation}
\end{equation}
\end{corollary}

\begin{proof}
Let $\eta_k=y^\varepsilon_k-y_k$.
Linearizing Eq.~(\ref{eq:discrete_perturbed_sampler}) around $y_k$ gives
\begin{equation}
\eta_{k+1}
=
J F_k(y_k)\eta_k
+
\varepsilon U_k(y_k)
+
\mathcal{O}(\varepsilon^2),
\qquad
\eta_0=0,
\end{equation}
where $\eta_k=\mathcal{O}(\varepsilon)$.
Repeated substitution yields Eq.~(\ref{eq:discrete_trajectory_variation}).
\end{proof}

Equation~(\ref{eq:discrete_trajectory_variation}) is the finite-step counterpart of
Eq.~(\ref{eq:trajectory_variation}): a local sparse intervention is first pushed through the
remaining U-Net computation at its current step, then propagated through all subsequent denoising
steps to the final sample.

\paragraph{Interpretation.}
The theory separates three objects:
\begin{equation}
\boxed{
D_j
\quad\longrightarrow\quad
\psi_{j,b,t,c}
=
\kappa_t\sigma_bJ_{b,t,c}D_j
\quad\longrightarrow\quad
\Phi\,\psi_{j,b,t,c}.
}
\end{equation}
The first is the shared sparse feature coordinate, the second is its context-dependent local effect on
the denoising vector field, and the third is the effect after transport to the final sample.
SharedSAE therefore does not assume that one decoder column corresponds to a globally fixed
generative direction. Instead, it supplies a reusable hidden-space coordinate system whose local
causal effect is determined by the current block, timestep, state, conditioning, and downstream
network Jacobian.

\end{document}